\documentclass{article}
\usepackage{ijcai24}

\usepackage{times}
\usepackage{soul}
\usepackage{url}
\usepackage[hidelinks]{hyperref}
\usepackage[utf8]{inputenc}
\usepackage[small]{caption}
\usepackage{graphicx}
\usepackage{amsmath}
\usepackage{amsthm}
\usepackage{booktabs}
\usepackage{algorithm}
\usepackage{algpseudocode}
\usepackage[switch]{lineno}
\usepackage{xcolor}
\usepackage{amsfonts}
\usepackage{multirow}
\title{Inferring Implicit Goals Across Differing Task Models}

\author{
Silvia Tulli$^1$
\and
Stylianos Loukas Vasileiou$^2$\and
Mohamed Chetouani$^{1}$\And
Sarath Sreedharan$^4$\\
\affiliations
$^1$Institute of Intelligent Systems and Robotics (ISIR) - CNRS - INSERM - Sorbonne University\\
$^2$McKelvey School of Engineering at Washington University in St. Louis\\
$^3$Department of Computer Science at Colorado State University\\
\emails
silvia.tulli@sorbonne-universite.fr,
v.stylianos@wustl.edu,
mohamed.chetouani@sorbonne-universite.fr,
sarath.sreedharan@colostate.edu
}

\newcommand{\BibTeX}{\rm B\kern-.05em{\sc i\kern-.025em b}\kern-.08em\TeX}

\theoremstyle{definition}
\newtheorem{defn}{Definition}
\newtheorem{prop}{Proposition}

\begin{document}
\newcommand{\querycost}[1]    {\ensuremath{\mathcal{C}(#1)}}
\newcommand{\expquerycost}[1]    {\ensuremath{\mathcal{Q}(#1)}}
\newcommand{\knownsubgoals}{\ensuremath{\mathcal{K}_{\mathcal{I}}}}
\newcommand{\queryvalue}[1]    {\ensuremath{\mathcal{V}(#1)}}
\newcommand{\queryprobin}[2]    {\ensuremath{\mathcal{P}(#1 \in #2)}}
\newcommand{\queryprobnotin}[2]    {\ensuremath{\mathcal{P}(#1\notin #2)}}



\maketitle
\begin{abstract}

One of the significant challenges to generating value-aligned behavior is to not only account for the specified user objectives but also any implicit or unspecified user requirements. The existence of such implicit requirements could be particularly common in settings where the user's understanding of the task model may differ from the agent's estimate of the model. Under this scenario, the user may incorrectly expect some agent behavior to be inevitable or guaranteed. This paper addresses such expectation mismatch in the presence of differing models by capturing the possibility of unspecified user subgoal in the context of a task captured as a Markov Decision Process (MDP) and querying for it as required. Our method identifies bottleneck states and uses them as candidates for potential implicit subgoals. We then introduce a querying strategy that will generate the minimal number of queries required to identify a policy guaranteed to achieve the underlying goal. Our empirical evaluations demonstrate the effectiveness of our approach in inferring and achieving unstated goals across various tasks.
\end{abstract}
\section{Introduction}

Humans often omit details they consider obvious, unavoidable, or not worth mentioning when providing instructions. This omission leads to implicit goals and unstated preferences that AI systems must navigate to truly align with human intentions. 
One potential source of such unstated subgoals or preferences could be behaviors that the human user setting the agent objectives may identify as inevitable.
The user would never bother stating anything regarding such behaviors, since they believe that it cannot be avoided.
One class of such behavior is that of visiting {\em bottleneck states} in the context of goal-based Markov Decision Process (MDP).
Here, bottleneck states refer to environment states that the agent must pass through to reach the stated goal.
In many cases, the human user may want the agent to pass through or visit some bottleneck states in addition to the goal, thus forming a set of intermediate subgoals.
However, the user may never specify them since, as far as the user is concerned, every trace of non-zero probability that leads to the goal passes through all the bottleneck states.
This should be all well and good, provided the human bottleneck states are also bottleneck states for the agent. Otherwise, the agent must make an effort to figure out what the user's underlying subgoals may be.


To see how such problems may arise, consider an agent tasked with guiding a tourist to a famous art museum. The tourist simply says, ``Get me a plan to get to the art museum," unaware of the city's metro system and expecting an above-ground route passing certain landmarks. The agent, however, might plan a route using the metro system.
For the agent's metro route, bottlenecks migh include entering the metro, making transfers, and exiting at the correct station. For the tourist's expected route, they might include crossing a river and passing through the city center.
This misalignment stems from differing world models: the agent's comprehensive transit data versus the tourist's limited knowledge of the city's layout. The challenge in AI alignment lies in bridging this gap – identifying and accounting for implicit aspects of the task that weren't mentioned.
This paper explores how an agent can learn and achieve the implicit subgoals of another agent, particularly when their understandings of the environment differ.

Traditional approaches to goal-conditioned reinforcement learning and planning often rely on explicit goal specification and assume shared world models between agents \cite{ijcai2022p770,chanesane2021goal}. This means trying to purely optimize for the achievement of specified human goals could
lead to suboptimal or even incorrect task execution in complex, real-world scenarios. 
We propose a novel approach that 
introduces and formalizes a specific type of
implicit subgoals within the 
MDP framework. Our method aims to compute policies that align with implicit subgoals by utilizing incomplete knowledge about another agent's model of the world. We formalize the problem using two distinct MDPs: the executing agent's model, the robot, and the goal-setting agent's model, the human. 
We will use the bottlenecks in the human model as a basis for potential hypotheses about the implicit subgoals of the human.
When potential implicit subgoals cannot be achieved, we will also make use of minimal querying to refine the agent's hypotheses about the human subgoals.

To evaluate our approach, we propose empirical evaluations in simulated scenarios. These evaluations will assess the ability to infer and align to implicit subgoals across various tasks.

Our key contributions include:
\begin{itemize}
    \item Formulation of implicit subgoals within an MDP framework
    \item Analysis of goal misalignment in the context of MDPs
    \item Introduction of methods to reason about and infer unstated goals across differing world models
    \item Development of novel querying algorithms for efficient information gathering to refine goal understanding
    \item Creation of a framework to learn and satisfy implicit subgoals using MDPs and information-theoretic querying
    \item Proposal of empirical evaluations in simulated scenarios to assess the ability to infer and align implicit subgoals across various tasks
\end{itemize}



\begin{figure*}[hbtp]
    \centering
    \includegraphics[width=\linewidth]{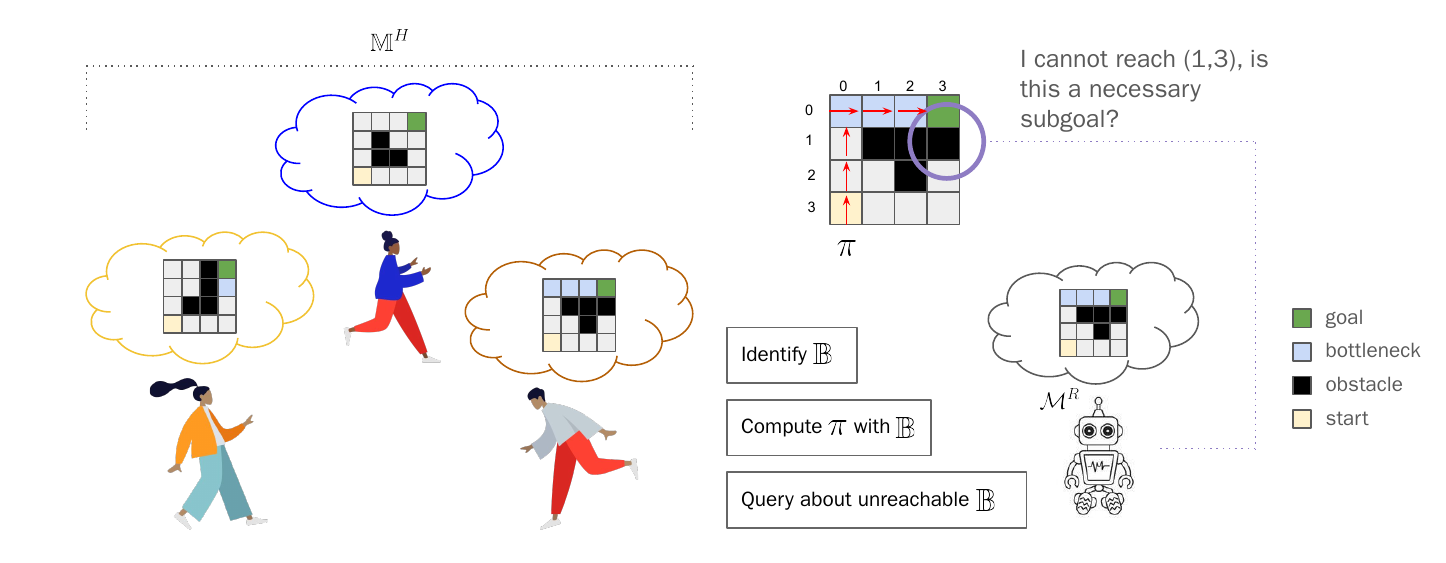}
    \caption{Bottleneck states are critical waypoints essential for reaching the goal in a given world model. Given a set of humans' world models $\mathbb{M}^{H}$, the robot has to compute a policy $\pi$ accounting for humans' $\mathbb{B}$, as they might be candidates for human implicit subgoals. Whenever the robot cannot reach a human's bottleneck due to discrepancies in world models, it queries whether this bottleneck is in fact a human subgoal.}
    \label{fig:implicit-goals}
\end{figure*}

\section{Background}
The problems studied in the paper will be ones that can be modeled as infinite horizon discounted Markov Decision Processes (MDP), particularly ones where the primary objectives correspond to achieving some goals. Here, goal achievement corresponds to visiting specific goal states.
The algorithms used in the paper will also exploit more traditional MDPs with rewards; as such, it is worth providing a more general definition. Specifically, an infinite horizon MDP is a tuple $\mathcal{M} = \langle S, A, T, s_0, \gamma, R \rangle$, where $S$ is the state space, $A$ is the action space, $T: S \times A \times S \rightarrow [0,1]$ is the transition function  (e.g., $ T(s_i,a_i,s^{\prime})$ provides the probability of transitioning from state $s$ to $s'$ under action $a$), $R: S \to \mathbb{R}$ is the reward function, $s_0 \in S$ is the initial state of the agent, and $\gamma \in [0,1)$ is the discount factor. Note that we will generally limit our attention to models where both $S$ and $A$ are finite sets.

In this setting, the solution takes the form of a deterministic, stationary policy $\pi: S \rightarrow A$ mapping states to actions. A value of a policy $\pi$, denoted as $V^\pi: S \to \Re$ provides the expected cumulative discounted reward obtained by following the policy from a given state. 
A policy is considered optimal if no policy with a value higher than the current one exists. 

In this paper, we focus on goal-directed problems, where the set of goal states is given as $S_G \subseteq S$.
In such problems, the reward function is sparse, i.e., it returns a small positive value for all states in $S_G$, and $0$ otherwise. The states in $S_G$ are treated as an absorbing state in that all transitions out of that state have zero probability. In such cases, we will represent the MDP equivalently in goal terms by dropping the reward function, i.e., $\mathcal{M} = \langle S, A, T, s_0, \gamma, S_G \rangle$.

A concept that we will leverage throughout the paper is that of goal-reaching traces. A goal-reaching trace of a policy from a state $s$, denoted as $\tau \sim_{\mathcal{M}} \pi(s)$, where $\tau = \langle s, \pi(s), ..., s_k \rangle$, corresponds to a state-action sequence with non-zero probability that terminates at a goal state. Since the reward is only provided by the goal, the value of the policy in a state is directly proportional to the probability of reaching the goal state under the given policy. We will capture this by the notation $P_G(s|\pi)$, which is given as

\[
P_G(s|\pi) = \sum_{\tau\sim\pi(s)} P(\tau|\pi),
\]

where $\tau$ are possible traces ending in a goal state and $P(\tau|\pi)$ is its likelihood under the given policy $\pi$.

\section{Planning for Implicit Subgoals}
As mentioned before, we are interested in a scenario where the robot's model of the task  $\mathcal{M}^R = \langle S, A, T^R, s_0, \gamma, S_G\rangle$ is different from the user's beliefs about the task $\mathcal{M}^H = \langle S, A, T^H, s_0, \gamma, S_G \rangle$ in terms of the underlying transition function\footnote{Even though we use the term robot to refer to the AI system, nothing in our framework requires the AI system to be physically embodied. Similarly, the user, in theory, could be any principal-agent setting the objective and not necessarily a human user.}. This difference leads the user to overlook specifying subgoals they are confident are bound to happen. 
We will leverage the notion of a {\em bottleneck state}, where a state is said to be a bottleneck state if it is reached in every path from the initial state to the goal. More formally,
\begin{defn}
\label{def:bottle}
 For a given MDP model $\mathcal{M} = \langle S, A, T, s_0, \gamma, S_G \rangle$, we define a \textbf{bottleneck state} as a state $s \in S$ that must be visited by any valid trace starting from $s_0$, for any policy, with non-zero probability of reaching a state in $S_G$. We denote the set of all bottlenecks as $B \subseteq S$. 
\end{defn}

Now, we assume that implicit subgoals ($\mathcal{I}_{G}$) are a subset of $B$ that the user wants the robot to achieve on its way to the goal. From the user's point of view, there is no need to spell out these implicit subgoals since they are bottleneck states in their own model and thus cannot be avoided. However, as the robot and user models differ, what is a bottleneck state in one model may not be in the other. What are interested in here is finding a policy that achieves the subgoals in the robot's models. Formally,
 
\begin{defn}
   For a given robot model $\mathcal{M}^R = \langle S, A, T^R, s_0, \gamma, S_G\rangle$, a policy $\pi$ is said to \textbf{achieve a set of implicit subgoal(s)}  $\mathcal{I}_{G}$, i.e., $\pi\models_{\mathcal{M}^R} \mathcal{I}_{G}$, if every goal reaching trace from $s_0$ $\tau \sim_{\mathcal{M}^R} \pi(s_0)$ passes through every state $s \in \mathcal{I}_{G}$. 
\end{defn}
The main challenge we face is to find a policy that achieves all the implicit subgoals when we do not know the implicit subgoals or the exact user model.
To achieve this, we make the following two reasonable assumptions: (1) We are given a set of potential user models $\mathbb{M}^{H}$ corresponding to different transition functions; and (2) We can query the user about potential implicit subgoals.
We will represent the latter using an oracle function of the form $\mathcal{O}_{\mathcal{I}_G}: S \rightarrow \{0,1\}$, where it can tell you whether a state $s$ is an implicit subgoal or not, i.e.,
\[\mathcal{O}_{\mathcal{I}_G}(s) = 
\begin{cases}
    1~\textrm{if}~s\in \mathcal{I}_G\\
    0~\textrm{otherwise}
\end{cases}
\]
Now, it is worth looking at how reasonable these assumptions are. It is always possible to have access to $\mathbb{M}^{H}$ by taking into account all possible well-formed transition functions that could be formulated for $S$ and $A$. This might result in an infinite set, but one can still formulate it. Similarly, for most cases of interest, such as when a human user is specifying a goal, they would also be able to identify whether or not a state is something they want to achieve.
Note that we are not claiming that the user can formulate an explicit set of all implicit subgoals upfront, but rather they can recognize states they want to achieve. While the former has been shown to fail \cite{Mechergui_Sreedharan_2024}, the latter is the backbone of all methods for learning from user preferences \cite{reddy2020,biyik2024}.
To make this assumption even more realistic, we will set the cost of querying the user extremely high. The problem then becomes to find the minimal set of queries we need to use to find a policy that achieves the true implicit subgoals.
\begin{defn}
    For a given robot model $\mathcal{M}^R$, a set of possible user models $\mathbb{M}^{H}$, and an oracle $\mathcal{O}_{\mathcal{I}_G}$ for an unknown implicit subgoal set $\mathcal{I}_G$, the problem of \textbf{query identification for implicit subgoals} is to choose states to pass to the oracle, so that the robot can identify a policy $\pi$ that is guaranteed to achieve the implicit subgoal $\mathcal{I}_G$, or that it can never come up with such a policy. 
\end{defn}

Not that we are not interested in finding any series of queries but an optimal one. After all, asking about all states $S$ will trivially satisfy this requirement. In this paper, we will focus on minimizing the expected cost of querying \cite{Bertsekas2010DynamicPA}. Here, each query carries a cost to ask, and the querying only ends once the agent can establish with certainty whether or not a policy exists that can satisfy all the human implicit subgoals along with the original goal. In this case, one can recursively define the expected cost for a given set of potential and known implicit goals. The optimal query at any point would correspond to the one that minimizes the expected value.
\begin{defn}
\label{def:optimal}
    At a given instance where the existence of the policy cannot be determined, with a set of remaining bottlenecks $\mathcal{B}$ and a set of known implicit subgoals \knownsubgoals, a state query $s_q$ is optimal if it minimizes the expected query cost for the bottleneck states and known sub-goals, i.e., where the expected query cost for a state $s$ is given as
\begin{align*}
\expquerycost{(\mathcal{B}, \knownsubgoals), s} = \querycost{s} +  (\queryprobin{s}{\mathcal{I}_G} * \queryvalue{\mathcal{B}\setminus \{s\}, \knownsubgoals \cup \{s\})} + \\
\queryprobnotin{s}{\mathcal{I}_G} * \queryvalue{(\mathcal{B}\setminus\{s\}, \knownsubgoals)}),
\end{align*}
    where $\queryvalue{(\mathcal{B}', \kappa')} = 0$, if $\kappa'$ is unachievable or if the set $\mathcal{B}' \cup \kappa'$ is achievable, else $\queryvalue{(\mathcal{B}', \kappa')} = \min_{s} \expquerycost{(\mathcal{B}', \kappa'), s}$
\end{defn}
As it should be clear from the form, we can map the problem itself into finding an optimal policy in an MDP. We will introduce such a query MDP in Section \ref{sec:query}. 

\section{Query Identification for Implicit Subgoals Set}
\label{sec:query}
Next, we will look at algorithms that will allow us to identify such queries.
Our primary approach here would involve first finding a set of potential implicit subgoals and then querying the user until we are able to whittle it down to a set that can be simultaneously achieved. 

Moving onto the problem of identifying potential implicit subgoals, we know that implicit subgoals are a subset of bottleneck states in the human model. Unfortunately, here, we might be given an infinite set of potential human models. However, in this setting, we can exploit the fact that bottleneck states are preserved through all outcome determinization \cite{Keller_Eyerich_2011}. Under all outcome determinization, every possible stochastic transition under an action is converted into a separate deterministic action. More formally, we construct a determinized version of a given model as follows:

\begin{defn}
    For a given MDP model $\mathcal{M} = \langle S, A, s_0, T, S_G\rangle$, a \textbf{determinized model} $\delta(\mathcal{M})$ is given as $\delta(\mathcal{M}) = \langle S, A', s_0, T', S_G \rangle$, such that for every non-zero transition $T(s_i,a_i,s^\prime)$ in $\mathcal{M}$, there exists a new action $a_i' \in A'$, such that $T(s_i,a_i',s^\prime)=1$. 
\end{defn}
This brings us to the first proposition that asserts that bottleneck states are preserved over determinization.
\begin{prop}
    Given a model $\mathcal{M}$ and its determinization $\delta(\mathcal{M})$, a state $s$ is a bottleneck state for $\mathcal{M}$ if and only if it is a bottleneck state in $\delta(\mathcal{M})$. 
\end{prop}
The proof of this proposition follows from the fact that any goal-reaching trace in $\mathcal{M}$ is a goal-reaching trace in $\delta(\mathcal{M})$, and vice versa. Thus, following Definition \ref{def:bottle}, the bottleneck states are the same across the two models.
It is easy to see that the number of possible unique determinized models for finite state and action sets is finite. After all, the maximum number of actions possible in any state is $|S|$. 
This brings us to the next proposition, which will establish the fact that the set of all determinized models (even if the original set isn't) is finite and considering determinization will never result in us looking at more models.
\begin{prop}
    Let $\mathbb{M}^H$ be the set of potential human models, where $S$ and $A$ are finite.  $\delta(\mathbb{M}^H)$ be given as
    \[\delta(\mathbb{M}^H) = \{\delta(\mathcal{M})\mid \mathcal{M} \in \mathbb{M}^H\}\]
    Then $\delta(\mathbb{M}^H)$ is finite and $|\delta(\mathbb{M}^H)| \leq |\mathbb{M}^H|$.
\end{prop}
The first property comes from the fact that the set of all possible determinized models is finite. The second one follows from the fact that determinization is an inherently surjective function.

The previous results show that, even if we start with infinite potential human models, we can determine the bottleneck states over a finite set of determinized models. We need a procedure to identify whether a state is a bottleneck state in a determinized model. To do this, we will create a modified MDP, where passing through the tested state is penalized while reaching the goal is given a small positive value. Now, if the state under test is not a bottleneck state, the optimal policy would involve avoiding the target state when possible. As such, we can assert it is not a bottleneck state when the optimal value for state $s_0$ is greater than zero. More formally,
\begin{prop}
\label{prop:avoid}
    For a determinized model $\delta(M) = \langle S, A, T, s_0, \gamma, S_G\rangle$, and for a target state $s_i$, we create a new MPD, $\mathcal{M}^{s_i} = \langle S, A, T, s_0, \gamma, R^{s_i}_{S_G}\rangle$, such that
    \[R(s) = \begin{cases}
        n ~\textrm{when}~s=s_i\\
        p ~\textrm{when}~s\in S_G\\
        0 ~\textrm{otherwise}
    \end{cases}\]
    where $n<0$, $p>0$ and $|n|>>p$. Here $s_i$ is not a bottleneck state, if and only if $V^*(s_0) > 0$ under the optimal policy for $\mathcal{M}^{s_i}$. 
\end{prop}
The validity of the proposition follows from the fact that if state $s_i$ is not a bottleneck state, then there should exist a path from $s_0$ to the goal that doesn't pass through $s_i$, which should result in a positive value for $s_0$. 

With this procedure in place, we can collect all the bottleneck states for each determinized model in the set $\delta(\mathbb{M}^H)$ by testing each state on each model. The union over these states represents our initial hypotheses set for the implicit goals ($\mathcal{H}^0_{\mathcal{I}}$). Our next objective will be to identify the set of maximal subsets of $\mathcal{H}^0_{\mathcal{I}}$, for which the robot can generate policies that achieve them. However, before we discuss the search procedure to find maximal subsets, we need a procedure that can identify policies that achieve subgoal when one exists. We will again convert it to that planning over MDPs. In this case, this will involve planning over two different planning problems. Firstly, one that will only count goal achievement if the trace passes through all the subgoals. 

\begin{prop}
    For an MDP $\mathcal{M} = \langle S, A, T, s_0, \gamma, S_G\rangle$, and a subgoal set $\hat{S}$, we create a new MDP $\mathcal{M}^{\hat{S}} = \langle S^{\hat{S}}, A, T^{\hat{S}}, s_0, \gamma, R^{\hat{S}}\rangle$, where
    \begin{itemize}
        \item $S^{\hat{S}}$ - New state space. Obtained by copying over the original states and having a copy for reaching the state after passing through each subgoal, i.e., $|S^{\hat{S}}| = |S|* |2^{\hat{S}}|$. A way to conceptualize the new state space is to assume we are adding $|\hat{S}|$ binary features to each state that track whether specific subgoals were visited previously.
        \item $T^{\hat{S}}$ - Identical to the previous transition function, but now every outgoing transition from a state in $\hat{s} \in \hat{S}$ turns the corresponding feature for $\hat{s}$ true in the resultant state.
        \item $R^{\hat{S}}$ - The new reward function. This function only returns a positive value for copies of previous goal states where all with features for all the states in $\hat{S}$ turned true. A negative reward is returned for all other copies of the goal state.
    \end{itemize}
    If there exists a policy $\pi$ that achieves the subgoal set, then there exists a policy $\hat{\pi}$, such that all goal reaching trace for $\hat{\pi}$ exits at the goal state copies where all the features for subgoals are true.
\end{prop}
Now, this proposition is true because any trace that ends in a goal state without passing through all the subgoal states will provide a negative reward. However, one might not be able to tell if the policy identified corresponds to such a policy or not. 
One can only do so by running a test over the determinized version of $\mathcal{M}^{\hat{S}}$, similar to the one described in Proposition \ref{prop:avoid}. 
One difference is that the actions in each state are limited to the one listed by $\hat{\pi}$. Here the test is run for each potential state in $\hat{S}$.

Now, with these tests in place, we can start searching for the maximally achievable subsets over the union of all bottleneck states across the potential human models.


\begin{algorithm}[!ht] 
\caption{Find the set of maximally achievable subsets of the set of all possible human bottleneck states.}
\begin{algorithmic}[1]
\State \textbf{Input:} $\mathcal{M}^R$, $B$ 
\State \textbf{Output:} Set of maximal achievable subsets $I$

\Function{FindMaximalAchievableSubsets}{$\mathcal{M}^R$, $B$}
    \State \Return GenerateSubsets(0, $\emptyset$, $B$, $\mathcal{M}^R$)
\EndFunction

\Function{GenerateSubsets}{index, current\_subset, B, $\mathcal{M}^R$}
    \If{index = |B|} 
        \If{¬CheckAchievability(current\_subset, $\mathcal{M}^R$)}
            \State \Return $\emptyset$
        \EndIf
        \State maximal\_subset $\leftarrow$ current\_subset
        \For{i from $|\textrm{current\_subset}|$ to $|B|$ - 1}
            \State new\_subset $\leftarrow$ maximal\_subset $\cup$ {B[i]}
            \If{CheckAchievability(new\_subset, $\mathcal{M}^R$)}
                \State maximal\_subset $\leftarrow$ new\_subset
            \EndIf
        \EndFor
        \State \Return {maximal\_subset}
    \EndIf
    
    \State result $\leftarrow$ result $\cup$ GenerateSubsets(index + 1, current\_subset $\cup$ {B[index]}, B, $\mathcal{M}^R$)
    \State \Return result
\EndFunction
\end{algorithmic}
\label{algo:maximal_subsets}
\end{algorithm}
Algorithm \ref{algo:maximal_subsets} provides the pseudo-code for finding maximal achievable subsets. It follows a recursive depth-first approach with pruning, where the function GenerateSubsets systematically explores combinations of bottleneck states. Starting from an empty set, at each recursive step, it makes two choices: either include or exclude the current state indexed by ``index". When an achievable subset is found, it checks if adding any remaining element makes it unachievable to confirm maximality. The algorithm maintains efficiency through early termination - if a subset is unachievable, all its supersets are pruned from consideration. All maximal achievable subsets are collected in $I$. By using caching and systematic exploration, the algorithm identifies all maximal subsets that can be achieved by a policy in $\mathcal{M}^R$, avoiding the exploration of redundant or unnecessary combinations.
Here, we can further improve the efficiency by doing a simple pre-test before generating a policy for $\mathcal{M}^{\hat{S}}$. We can first test whether there exists a path in the determinized model that goes through all the states in the current set. If it doesn't, we can directly mark it as being unachievable.

With the identification of $\mathbb{I}$, we are finally set to identify the optimal querying strategy. The problem of finding a querying strategy will be converted into an MDP planning problem.
\begin{defn}
    For a set of potentially achievable subgoals $\mathbb{I}$, selected from a bottleneck set $\mathbb{B}$, the query MDP is defined as $\mathcal{M}^{Q} = \langle S^Q, A^Q, T^Q, s_0^Q, \gamma, R^Q\rangle$, where each component is defined as follows.
    \begin{itemize}
        \item $S^Q$ - Here, each state consists of subgoals known to be part of the human implicit subgoals and those known to be not part of it, i.e.,  $S^Q = 2^{\mathbb{B}} \times 2^{\mathbb{B}} $. Each state is represented as $(K_{\mathcal{I}}, K_{\neg\mathcal{I}})$, where $K_{\mathcal{I}}$ are the states known to be implicit subgoals and $K_{\neg\mathcal{I}})$ the ones known to be not.
        \item $A^Q$ - One action to query about each element in $\mathbb{B}$, i.e., $|A^Q| = |\mathbb{B}|$. The action space is pruned by eliminating queries to already classified bottlenecks, reducing the branching factor of the search.
        \item $T^Q$ - A transition function that replicates potential outcomes of the oracle and determines the absorber state:
        \begin{itemize}
            \item Here, there are two classes of absorber states that enable early termination. One where the current state belongs to one of the achievable maximal set, i.e., a state $(K_{\mathcal{I}}, K_{\neg\mathcal{I}})$, where the potential subgoals $\hat{\mathcal{I}} =  K_{\mathcal{I}} \cup (\mathbb{B}\setminus K_{\neg\mathcal{I}})$ are of the form that $\hat{\mathcal{I}} \subset \mathcal{I}'$, for some $\mathcal{I}' \in \mathbb{I}$. In the second, the known implicit subgoals are not part of any achievable subgoal set, i.e., $\not\exists \mathcal{I}' \in \mathbb{I}$, such that $K_{\mathcal{I}} \subseteq \mathcal{I}'$. This enables early termination of the search when either an unachievable bottleneck is found to be necessary or all bottlenecks are classified.
            \item If $a$ corresponds to a state $s$ not part of the known or unknown part of the state, it will transition to $(K_{\mathcal{I}} \cup{s}, K_{\neg\mathcal{I}})$ or  $(K_{\mathcal{I}}, K_{\neg\mathcal{I}} \cup{s})$ with equal likelihood.
            \item If the $a$ corresponds to a state that is already part of $K_{\mathcal{I}}$ or $K_{\neg\mathcal{I}}$, it leaves the state unchanged.
        \end{itemize}
        \item $s_0^Q$ - start state where nothing is known, i.e., $s_0^Q = (\{\},\{\})$.
        \item $R^Q$ - The reward function. For any non-absorbing state it returns a heavy penalty associated with query $C^Q < 0$. For the absorbing state, if it is an unachievable state, it returns zero, and for the achievable state, it returns a scaled positive value $p_{\mathcal{I}}$ proportional to the value of the initial state in the corresponding optimal policy for $\mathcal{M}^{\hat{S}}$. The positive values are selected such that if the optimal value returned for the $K_{\mathcal{I}_1}$ is higher than $K_{\mathcal{I}_2}$, then $p_{\mathcal{I}_1} > p_{\mathcal{I}_2}$. Also, for all $p_{\mathcal{I}_i}$, we require that it is smaller than the magnitude of the penalty from a single query, i.e., $|C^Q| > p_{\mathcal{I}_max}$, where $p_{\mathcal{I}_max}$ highest reward possible for any achievable state.
        \item $\gamma$ - is set to a high discount factor so as not to overlook the cost of future queries too much.
    \end{itemize}
\end{defn}
\begin{prop}
    The optimal policy identified for the MDP $\mathcal{M}^{Q}$, corresponds to an optimal query strategy described in Definition \ref{def:optimal}
\end{prop}
This follows from the structure of the MDP. The cost and transition function, here, are selected so the Bellman equations for the MDP replicate the optimality equations referred to in Definition \ref{def:optimal} (with min replaced with max to reflect the switch from costs to rewards). One point of departure here from the earlier definition is the allocation of positive rewards to absorber states where the reward is proportional to its value associated in the model $\mathcal{M}^{\hat{S}}$. Given the role played by 
discount factor, this means that a higher value is associated with bottleneck subsets where the goals and subgoals are achieved over shorter traces. This means that the problem of finding subgoal sets that are easier to achieve becomes a secondary objective for the MDP. However, please note that the larger penalty for the query cost means that this secondary objective will never be pursued at the cost of a potentially larger number of queries.\\

Now, even though there are efficient MDP solvers, solving the above MDP could be computationally expensive if there exists a large number of possible bottleneck states. However, it is possible to show that we can actually build a smaller MDP that first filters out all the bottleneck states that cannot be achieved in the robot model and create a new query MDP only containing the remaining states. We will refer to the resulting MDP as the pruned query MDP (and represent it as $\hat{\mathcal{M}}^Q$). Now, we will create a meta query policy that will first query about all non-achievable bottlenecks before switching over to the optimal policy for the pruned query MDP ($\hat{P}^Q$). We will refer to this new modified query as $\Pi^Q$, and it is defined as
\begin{align*}
  \Pi^Q((K_{\mathcal{I}}, K_{\neg\mathcal{I}})) =
  \begin{cases}
s_i \textrm{ if } |\mathbb{B}^{\empty}\setminus(K_{\mathcal{I}} \cup K_{\neg\mathcal{I}})| < 0,\\
~~~~~~~~~~\textrm{where } s_i \in \mathbb{B}^{\empty}\setminus(K_{\mathcal{I}} \cup K_{\neg\mathcal{I}})\\ \hat{P}^Q((K_{\mathcal{I}} \setminus \mathbb{B}^{\empty},K_{\neg\mathcal{I}} \setminus \mathbb{B}^{\empty})) ~~\textrm{Otherwise},
    \end{cases}
\end{align*}
where $\mathbb{B}^{\empty}$ is the set non-achievable bottleneck states. Even though such a pruning method could result in an exponential reduction in the state space of the MDP problem to be solved, we can show that this new policy is in fact optimal for the original query model. Or more formally,
\begin{prop}
    The meta policy $\Pi^Q$ is an optimal policy for the query MDP $\mathcal{M}^Q$.
\end{prop}
\begin{proof}[Proof Sketch]
    The primary proof for the above statement relies on establishing the fact that in a non-absorbing state for   $\mathcal{M}^Q$, the cost of querying a non-achievable bottleneck state is always going to cheaper than or equal to the cost of a query about a state that is part of some achievable subset of bottlenecks. 
    
    This can be shown by the fact that any query not involving a non-achievable state must involve at least one outcome, with at least one future query guaranteed to be required. 
    This guarantee follows from two facts. For such a query, at least one of the outcomes must contain a potentially achievable subset. Without such an outcome, the original query state would be unachievable and thus an absorbing state. As for the at least one guaranteed future query comes from the fact that since this query skipped over a non-achievable bottleneck, the achievability of the outcome can only be established after resolving whether or not the non-achievable bottleneck is part of the human implicit subgoal.

    Now, on the other hand, there can, at most, be one outcome where further querying is possibly required. In this case, future querying is also not guaranteed because, after the removal of the unachievable state, it could just result in a query state that corresponds to an absorber state for an achievable subset. 
    
    In other words, the query about a non-achievable bottleneck is always guaranteed to remove a future query from all outcomes. But for bottleneck states that can be achieved under some policy, we are guaranteed that we need to query about the non-achievable bottleneck at some point in the future. 

    Finally, there order in which the non-achievable states need to be queried doesn't matter as its not part of any achievable subsets and thus the order doesn't affect how any of the potentially achievable subsets can be queried. Finally, the secondary objective doesn't really affect this order, as it relates to reducing the expected number of queries and the secondary objective is dominated by the cost of the number of queries.  
\end{proof}

\section{Related Work}
Our work intersects with three primary areas of research: reward misspecification, planning with different world models (including insights from Theory of Mind), and query mechanisms in the context of assistance.
\subsubsection{Reward Misspecification}
The challenge of aligning AI systems with human values and intentions has been a growing concern in recent years. Reward misspecification, where the specified reward function does not fully capture the user's true objectives, has been identified as a critical issue in this domain. \cite{Hadfield2017} introduced the concept of inverse reward design, which aims to infer the true objective function from an observed reward function. 
This work highlights the importance of considering the context in which rewards are specified. Building on this, \cite{Majumdar2017RisksensitiveIR} proposed a risk-sensitive inverse reinforcement learning framework, addressing uncertainties in the reward function. More recently, \cite{shah2018the} explored the idea of preferences implicit in the design of the reward function, introducing methods to infer these implicit preferences. 
Similarly, \cite{gleave2021quantifying} investigated the problem of reward hacking, where agents exploit misspecified rewards in unexpected ways.
A significant contribution to this area is the work of \cite{Sreedharan2024HandlingRM}. This research directly addresses the challenge of reward misspecification when there is a discrepancy between the agent's and the user's expectations of the task or environment. The authors propose a framework that not only identifies potential misalignments in the reward function but also considers the differences in model understanding between the agent and the user. This approach is particularly relevant to our work as it bridges the gap between reward misspecification and model reconciliation. Our work extends these ideas by specifically focusing on the identification of unspecified subgoals, which may arise due to differences between the user's and agent's understanding of the task model.

\subsubsection{Planning with Different World Models}
The problem of agents and humans having different world models has been explored in various contexts within AI planning and human-robot interaction. \cite{chakraborti2017} introduced the concept of model reconciliation in the context of explainable planning. Their work focuses on generating explanations that bridge the gap between the agent's and the human's model of the world. This line of research was further developed by \cite{Sreedharan2018HierarchicalEM}, who explored the generation of minimal explanations for model reconciliation.
In a related vein, \cite{Bobu2018LearningUM} investigated the problem of learning from corrections in domains where the human and robot have different feature spaces. Recent work aims to address the representation misalignment \cite{peng2024adaptive}. Their work highlights the importance of considering model differences in human-robot interaction scenarios.
More specific to planning, \cite{Keren_Gal_Karpas_2014} introduced the notion of goal recognition design, which involves modifying the environment to facilitate goal recognition. While not directly addressing model differences, this work shares our interest in identifying key states that influence goal achievement.
Our approach builds upon these ideas by using bottleneck states as candidates for potential implicit goal sets, thereby addressing the challenge of planning in the presence of model differences and unspecified goals.
Research on Theory of Mind provides complementary insights into how agents plan with different world models. \cite{Ho2021CognitiveSA} highlight how humans use abstract, structured causal models of others' mental states to plan interventions effectively. Their work shows that Theory of Mind is particularly crucial when agents have different mental models, as in teaching or pragmatic communication scenarios. These insights complement approaches to model reconciliation - while current AI systems focus on explicit explanations to bridge model differences, humans appear to handle such differences through abstract representations that support efficient search over possible interventions \cite{Ho2022PlanningWT}.
This suggests that incorporating Theory of Mind-inspired representations could enhance AI systems' ability to plan effectively when faced with model differences.

By combining insights from both reward misspecification and model reconciliation research, our work offers a novel approach to generating value-aligned behavior in scenarios where the agent's and user's models may differ.

\subsubsection{Query Mechanisms and POMDP for assistance}
Query mechanisms enable robots to actively seek clarification or additional information from human operators when confronted with ambiguous or incomplete reward signals. Existing research in this area explores aspects such as the robot's partial observability of the human model \cite{Fern2007ADM}, learning a policy that aligns with the user's true goals \cite{Ng1999PolicyIU}, and the inherent ambiguity in defining a task through a reward function \cite{Abel2021OnTE}.
\cite{Fern2007ADM} introduce a decision-theoretic framework for intelligent assistance systems, focusing on Hidden Goal Markov Decision Processes (HGMDPs) as a model for selecting assistive actions. They conceptualize the assistance problem as a POMDP where the agent's goal is hidden from the assistant. This model enables the assistant to reason under uncertainty about the user’s objectives, providing a structured approach to selecting actions that account for varying task models and user preferences.

Reward shaping in reinforcement learning offers valuable formalisms for guiding learning without altering optimal policies. \cite{Ng1999PolicyIU}'s work established potential-based shaping functions, demonstrating their necessity and sufficiency for guaranteeing policy invariance. This principle is particularly relevant to designing effective query strategies. By leveraging these shaping functions, robots can refine their understanding of ambiguous tasks through queries without compromising the integrity of the learned policy.

\cite{Abel2021OnTE} investigate the expressivity of Markov reward functions in reinforcement learning, establishing a theoretical foundation for understanding the limits of reward-based task specification. 
They propose three abstract notions of "task" and
prove the existence of environment-task pairs for each task type that no Markov reward function can capture, highlighting fundamental limitations in using rewards to specify arbitrary tasks. 
These findings are relevant to designing query mechanisms that efficiently assess task expressibility. The identified limitations in reward expressivity underscore the potential need for more sophisticated query mechanisms in assistance scenarios, particularly when capturing complex task specifications.

Query mechanisms are also central to hierarchical reinforcement learning approaches. \cite{Nguyen2021LearningWA} employ a POMDP framework, using query mechanisms to ask humans about potential subgoals in a task. Extending this work, \cite{Zheng2023ProgressivelyEL} explores the use of natural language to facilitate learning across different environments. In their approach, the agent proactively communicates its planned actions and seeks feedback from humans on the quality of those plans. Their goal is to optimize language abstractions, minimizing the number of queries needed and thus reducing long-term communication efforts. However, in their approach, the human has full visibility of the robot's task model. Additionally, there can be mismatches in vocabulary—for example, the robot's understanding of ``bake a pork" might differ from the human's interpretation. We aim to address these discrepancies by designing a system that learns to query, taking into account that the human may have an incomplete or inaccurate understanding of the robot's task model.


\section{Evaluation}

We evaluate our approach on a set of standard Markov Decision Process (MDP) benchmarks.

\subsubsection{Environments}
The base environment is a Maze, a basic setting where agents navigate a grid with randomly placed obstacles, providing our baseline scenario with simple navigation challenges and clear bottleneck states at narrow passages. Building on this, Four-Rooms extends Maze by dividing it into four quadrants connected by doorways. The fixed room structure with randomized door placements creates natural bottlenecks, testing our method's ability to identify critical transition points. PuddleWorld introduces additional complexity by adding puddles that incur penalties when traversed. This environment forces trade-offs between path length and safety, creating interesting bottleneck scenarios where avoiding puddles competes with finding shortest paths. Finally, RockWorld features two types of rocks - valuable rocks that provide rewards when collected and dangerous rocks that incur penalties. This tests bottleneck identification in scenarios with resource management and risk-reward trade-offs.

\subsubsection{Methodology}
We conduct parallel experiments across different configurations for each environment\footnote{For reproducibility, the code and implementation details will be available on GitHub. The repository includes instructions to replicate all experiments.}. Our framework runs multiple independent trials with varying parameters including grid sizes chosen to balance complexity and tractability, obstacle percentages selected to ensure feasible paths while creating meaningful navigation challenges, and different numbers of human models to test scalability with varying levels of preference diversity. Each model uses a unique obstacle seed to ensure statistically independent trials. We set a query threshold of 1000 and run 3 trials per configuration.

\subsubsection{Results}
Our analysis reveals significant performance improvements across environments. For 4×4 grids, Strategic Query showed particularly strong results in Four Rooms ($p<0.001$) and Rocks ($p=0.012$), with query counts reduced from 3.7-4.8 (Query-All) to 2.0-3.3 queries, achieving reductions of 22-41\%. Query times remained efficient at 2-3 seconds. In 6×6 environments, while pruning times increased, efficiency gains persisted with query counts of 3.2-4.3 versus 3.7-7.7 for Query-All, yielding 13-40\% reductions, though with marginally significant differences (Puddle: $p=0.053$, Rocks: $p=0.073$).

Analysis across 20 human models with 10\% obstacle density demonstrates consistent performance, particularly in basic grid environments (35.6\% reduction). More complex environments like Rocks and Puddle maintained substantial improvements (26-33\% reduction). The stability of these improvements across varying model counts suggests robust scalability. Notably, Four Rooms showed the strongest statistical significance in 4×4 grids ($p<0.001$) while maintaining performance benefits in larger environments, albeit with reduced statistical significance ($p=0.411$ for $6×6$) (Figure: \ref{fig:results}).
\begin{figure}[t]
    \begin{center}
    \setlength{\tabcolsep}{4pt}
    \begin{minipage}[t]{0.45\textwidth}
        \includegraphics[width=\linewidth]{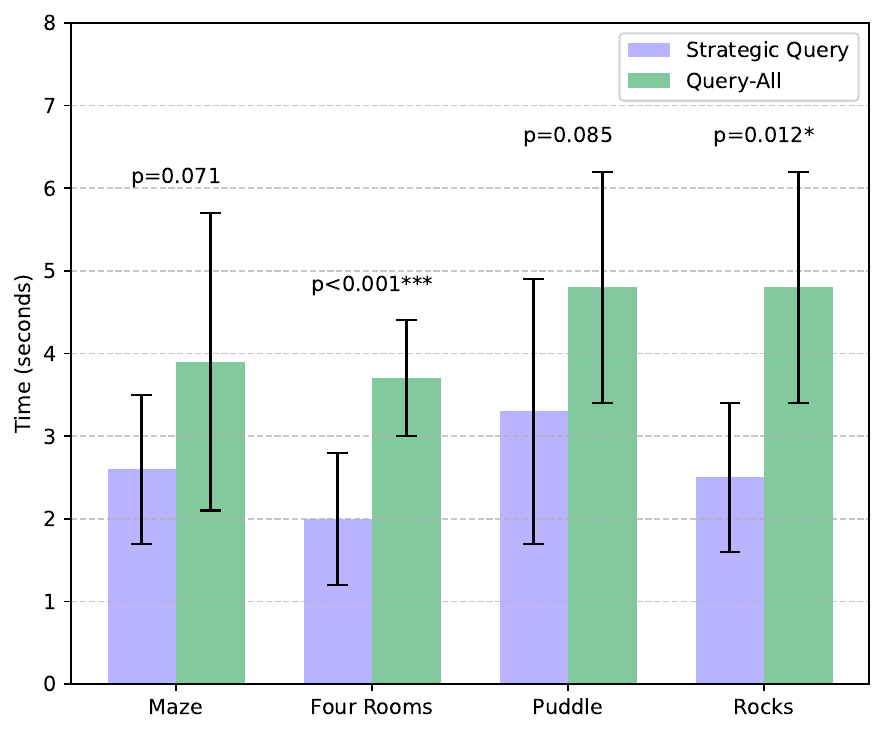}
    \end{minipage}
    \begin{minipage}[t]{0.45\textwidth}
        \small
        \centering
        \begin{tabular}{lcccc}
        \toprule
        Domain & Grid & Str. & Query & Red. \\
         & Size & Query & All & (\%) \\
        \midrule
        Maze & (4,4) & 2.6$\pm$0.9 & 3.9$\pm$1.8 & 22.0$\pm$30.1 \\
         & (6,6) & 3.8$\pm$1.5 & 6.8$\pm$3.9 & 35.6$\pm$33.6 \\[2pt]
        Four R. & (4,4) & 2.0$\pm$0.8 & 3.7$\pm$0.7 & 0.0$\pm$0.0 \\
         & (6,6) & 3.2$\pm$1.0 & 3.7$\pm$0.7 & 13.3$\pm$22.1 \\[2pt]
        Puddle & (4,4) & 3.3$\pm$1.6 & 4.8$\pm$1.4 & 26.7$\pm$29.7 \\
         & (6,6) & 4.3$\pm$3.3 & 7.7$\pm$4.0 & 40.7$\pm$29.6 \\[2pt]
        Rocks & (4,4) & 2.5$\pm$0.9 & 4.8$\pm$1.4 & 41.1$\pm$27.1 \\
         & (6,6) & 4.1$\pm$2.4 & 6.9$\pm$3.1 & 33.7$\pm$29.4 \\
        \bottomrule
        \end{tabular}
    \end{minipage}
    \end{center}
    \caption{On Top: Performance comparison between Strategic Query and Query-All approaches across four environments with 4×4 grid size. Results show mean execution times $\pm$ standard deviation based on 20 human preference models, 10\% obstacle density, and 3 runs per configuration with a query threshold of 1000. Table: Basic Performance Metrics showing query counts and reduction percentages.}
    \label{fig:results}
\end{figure}
Query times remained efficient at 2-3 seconds. In 6×6 environments, while pruning times increased, efficiency gains persisted with query counts of 3.2-4.3 versus 3.7-7.7 for Query-All.

Scaling to 8×8 grids further validates the effectiveness of our approach. Strategic Query maintained consistent query counts (3.2 ± 1.6) across environments while Query-All increased to 11.6 ± 2.9 (except Four Rooms: 3.2 ± 1.0), resulting in substantial reductions (71.9 ± 10.9\%) for Maze, Puddle, and Rocks environments. While pruning times increased to $\sim$308-325 seconds, the method achieved consistent speedup (1.05x) across all environments. Human bottleneck values stabilized around 5.8 ± 1.5 (Four Rooms: 1.6 ± 0.5), demonstrating reliable preference modeling despite increased environmental complexity.

Further analysis with varying parameters reveals interesting trends. With 10 human models and 0.1 obstacle density, bottleneck finding times remained consistent (1.7-1.8s) with total runtimes of 4.1-9.2s. Reducing to 5 human models improved computational efficiency (0.8-0.9s bottleneck finding, 2.1-4.0s total runtime). However, increasing obstacle density to 0.15 significantly impacted performance, particularly in environments like Maze and Rocks (27.4s and 13.7s bottleneck finding respectively), with total runtimes increasing to 13.1-50.7s and higher human bottleneck values (3.3-4.2). This demonstrates the method's stability with increased human models but sensitivity to environmental complexity at higher obstacle densities.

\section{Discussion}\label{sec:disc}
The paper presents a way a planning system can identify hidden subgoals of users, even when the human model may not be exactly known. We present algorithms to both identify potential candidates and generate an optimal number of queries.
We evaluate the effectiveness of the proposed method on a set of standard benchmark problems.
In terms of future work, one of the immediate next steps would be to run user studies. We plan to do them in realistic and everyday scenarios, possibly a robotic one, with significant population size. 
This would allow us to capture the effectiveness of our method in terms of the load placed on the humans and also test another related hypothesis.
For example, one could test whether people would be open to more queries if it significantly improves the agent efficiency.
This paper also focuses on the exact method that identifies optimal solutions.
It would be interesting to see if we could leverage approximate methods.
It would also be interesting to see if we could use other knowledge sources like pre-trained large language models, to get more information about user knowledge and preferences \cite{zhou2024establishing}.

\bibliographystyle{named}
\bibliography{ijcai24}

\begin{thebibliography}{}

\bibitem[\protect\citeauthoryear{Bobu \bgroup et al.\egroup}{2018}]{Bobu2018LearningUM}
Andreea Bobu, Andrea~V. Bajcsy, Jaime~Fernández Fisac, and Anca~D. Dragan.
\newblock Learning under Misspecified Objective Spaces.
\newblock In {\em Conference on Robot Learning}, 2018.

\bibitem[\protect\citeauthoryear{Bertsekas}{2010}]{Bertsekas2010DynamicPA}
Dimitri~P. Bertsekas.
\newblock Dynamic Programming and Optimal Control 3rd Edition, Volume II.
\newblock 2010.

\bibitem[\protect\citeauthoryear{Biyik \bgroup et al.\egroup}{2024}]{biyik2024}
Erdem Biyik, Nima Anari, and Dorsa Sadigh.
\newblock Batch Active Learning of Reward Functions from Human Preferences.
\newblock {\em J. Hum.-Robot Interact.}, 13(2):24, June 2024.

\bibitem[\protect\citeauthoryear{Chakraborti \bgroup et al.\egroup}{2017}]{chakraborti2017}
Tathagata Chakraborti, Sarath Sreedharan, Yu Zhang, and Subbarao Kambhampati.
\newblock Plan explanations as model reconciliation: moving beyond explanation as soliloquy.
\newblock In {\em Proceedings of the 26th International Joint Conference on Artificial Intelligence}, pages 156--163, Melbourne, Australia, 2017. AAAI Press.

\bibitem[\protect\citeauthoryear{Chane-Sane \bgroup et al.\egroup}{2021}]{chanesane2021goal}
Elliot Chane-Sane, Cordelia Schmid, and Ivan Laptev.
\newblock Goal-Conditioned Reinforcement Learning with Imagined Subgoals.
\newblock In {\em International Conference on Machine Learning}, 2021.

\bibitem[\protect\citeauthoryear{Fern \bgroup et al.\egroup}{2007}]{Fern2007ADM}
Alan Fern, Sriraam Natarajan, Kshitij Judah, and Prasad Tadepalli.
\newblock A Decision-Theoretic Model of Assistance.
\newblock In {\em International Joint Conference on Artificial Intelligence}, 2007.

\bibitem[\protect\citeauthoryear{Gleave \bgroup et al.\egroup}{2021}]{gleave2021quantifying}
Adam Gleave, Michael~D. Dennis, Shane Legg, Stuart Russell, and Jan Leike.
\newblock Quantifying Differences in Reward Functions.
\newblock In {\em International Conference on Learning Representations}, 2021.

\bibitem[\protect\citeauthoryear{Hadfield-Menell \bgroup et al.\egroup}{2017}]{Hadfield2017}
Dylan Hadfield-Menell, Smitha Milli, Pieter Abbeel, Stuart Russell, and Anca~D. Dragan.
\newblock Inverse reward design.
\newblock In {\em Proceedings of the 31st International Conference on Neural Information Processing Systems}, pages 6768--6777, Long Beach, California, USA, 2017. NIPS'17.

\bibitem[\protect\citeauthoryear{Ho \bgroup et al.\egroup}{2021}]{Ho2021CognitiveSA}
Mark~K. Ho and Thomas~L. Griffiths.
\newblock Cognitive science as a source of forward and inverse models of human decisions for robotics and control.
\newblock {\em ArXiv}, abs/2109.00127, 2021.

\bibitem[\protect\citeauthoryear{Ho \bgroup et al.\egroup}{2022}]{Ho2022PlanningWT}
Marcus~Wei Kei Ho, Rebecca Saxe, and Fiery~Andrews Cushman.
\newblock Planning with Theory of Mind.
\newblock {\em Trends in Cognitive Sciences}, 26:959--971, 2022.

\bibitem[\protect\citeauthoryear{Keller and Eyerich}{2011}]{Keller_Eyerich_2011}
Thomas Keller and Patrick Eyerich.
\newblock A Polynomial All Outcome Determinization for Probabilistic Planning.
\newblock In {\em Proceedings of the International Conference on Automated Planning and Scheduling}, pages 331--334, March 2011.

\bibitem[\protect\citeauthoryear{Keren \bgroup et al.\egroup}{2014}]{Keren_Gal_Karpas_2014}
Sarah Keren, Avigdor Gal, and Erez Karpas.
\newblock Goal Recognition Design.
\newblock In {\em Proceedings of the International Conference on Automated Planning and Scheduling}, pages 154--162, May 2014.

\bibitem[\protect\citeauthoryear{Liu \bgroup et al.\egroup}{2022}]{ijcai2022p770}
Minghuan Liu, Menghui Zhu, and Weinan Zhang.
\newblock Goal-Conditioned Reinforcement Learning: Problems and Solutions.
\newblock In {\em Proceedings of the Thirty-First International Joint Conference on Artificial Intelligence}, pages 5502--5511, July 2022.

\bibitem[\protect\citeauthoryear{Majumdar \bgroup et al.\egroup}{2017}]{Majumdar2017RisksensitiveIR}
Anirudha Majumdar, Sumeet Singh, Ajay Mandlekar, and Marco Pavone.
\newblock Risk-sensitive Inverse Reinforcement Learning via Coherent Risk Models.
\newblock In {\em Robotics: Science and Systems}, 2017.

\bibitem[\protect\citeauthoryear{Mechergui and Sreedharan}{2024a}]{Mechergui_Sreedharan_2024}
Malek Mechergui and Sarath Sreedharan.
\newblock Goal Alignment: Re-analyzing Value Alignment Problems Using Human-Aware AI.
\newblock In {\em Proceedings of the AAAI Conference on Artificial Intelligence}, pages 10110--10118, March 2024.

\bibitem[\protect\citeauthoryear{Mechergui and Sreedharan}{2024b}]{mechergui2024expectation}
Malek Mechergui and Sarath Sreedharan.
\newblock Expectation Alignment: Handling Reward Misspecification in the Presence of Expectation Mismatch.
\newblock In {\em The Thirty-eighth Annual Conference on Neural Information Processing Systems}, 2024.

\bibitem[\protect\citeauthoryear{Ng \bgroup et al.\egroup}{1999}]{Ng1999PolicyIU}
A.~Ng, Daishi Harada, and Stuart~J. Russell.
\newblock Policy Invariance Under Reward Transformations: Theory and Application to Reward Shaping.
\newblock In {\em International Conference on Machine Learning}, 1999.

\bibitem[\protect\citeauthoryear{Peng \bgroup et al.\egroup}{2024}]{peng2024adaptive}
Andi Peng, Belinda~Z. Li, Ilia Sucholutsky, Nishanth Kumar, Julie Shah, Jacob Andreas, and Andreea Bobu.
\newblock Adaptive Language-Guided Abstraction from Contrastive Explanations.
\newblock In {\em 8th Annual Conference on Robot Learning}, 2024.

\bibitem[\protect\citeauthoryear{Reddy \bgroup et al.\egroup}{2020}]{reddy2020}
Siddharth Reddy, Anca~D. Dragan, Sergey Levine, Shane Legg, and Jan Leike.
\newblock Learning human objectives by evaluating hypothetical behavior.
\newblock In {\em Proceedings of the 37th International Conference on Machine Learning}, page 743, 2020.

\bibitem[\protect\citeauthoryear{Shah \bgroup et al.\egroup}{2019}]{shah2018the}
Rohin Shah, Dmitrii Krasheninnikov, Jordan Alexander, Pieter Abbeel, and Anca Dragan.
\newblock The Implicit Preference Information in an Initial State.
\newblock In {\em International Conference on Learning Representations}, 2019.

\bibitem[\protect\citeauthoryear{Sreedharan \bgroup et al.\egroup}{2018a}]{sreedharan2018}
Sarath Sreedharan, Tathagata Chakraborti, and Subbarao Kambhampati.
\newblock Handling Model Uncertainty and Multiplicity in Explanations via Model Reconciliation.
\newblock In {\em Proceedings of the International Conference on Automated Planning and Scheduling}, pages 518--526, June 2018.

\bibitem[\protect\citeauthoryear{Sreedharan \bgroup et al.\egroup}{2018b}]{Sreedharan2018HierarchicalEM}
Sarath Sreedharan, Siddharth Srivastava, and S. Kambhampati.
\newblock Hierarchical Expertise-Level Modeling for User Specific Robot-Behavior Explanations.
\newblock {\em ArXiv}, abs/1802.06895, 2018.

\bibitem[\protect\citeauthoryear{Sreedharan and Mechergui}{2024}]{Sreedharan2024HandlingRM}
Sarath Sreedharan and Malek Mechergui.
\newblock Handling Reward Misspecification in the Presence of Expectation Mismatch.
\newblock {\em ArXiv}, abs/2404.08791, 2024.

\bibitem[\protect\citeauthoryear{Zhou \bgroup et al.\egroup}{2024}]{zhou2024establishing}
Sizhe Zhou, Sha Li, Yu Meng, Yizhu Jiao, Heng Ji, and Jiawei Han.
\newblock Establishing Knowledge Preference in Language Models.
\newblock 2024.

\bibitem[\protect\citeauthoryear{Zheng \bgroup et al.\egroup}{2023}]{Zheng2023ProgressivelyEL}
Ruijie Zheng, Khanh Nguyen, Hal Daum'e, Furong Huang, and Karthik Narasimhan.
\newblock Progressively Efficient Learning.
\newblock {\em ArXiv}, abs/2310.13004, 2023.

\end{thebibliography}

\end{document}